\documentclass[final,12pt]{colt2022} 

\title[]{Tracking Most Significant Arm Switches in Bandits}
\usepackage{times}

\coltauthor{%
 \Name{Joe Suk} \Email{js5338@columbia.edu}\\
 \addr Columbia University
 \AND
 \Name{Samory Kpotufe} \Email{samory@columbia.edu}\\
 \addr Columbia University
}

\usepackage[utf8]{inputenc} 
\usepackage[T1]{fontenc}    
\usepackage{hyperref}       
\usepackage{url}            
\usepackage{booktabs}       
\usepackage{amsfonts}       
\usepackage{nicefrac}       
\usepackage{microtype}      
\usepackage{lipsum}
\usepackage{graphicx}
\graphicspath{ {./images/} }
\usepackage[title]{appendix}
\usepackage{bbm}
\usepackage{setspace}
\usepackage{natbib,url}
\usepackage{wrapfig}
\usepackage{etoolbox}
\usepackage{bbding}
\usepackage{blindtext}

\date{}

\usepackage{amsmath}
\usepackage{amsfonts}
\usepackage{graphicx}
\usepackage{mathrsfs,amsmath}
\usepackage{amssymb}
\usepackage{titlesec}
\usepackage{hyperref}
\usepackage{etoolbox}
\usepackage{algorithm}
\usepackage{xcolor}
\usepackage[shortlabels]{enumitem}
\usepackage{thmtools}

\usepackage{chngcntr}

\DeclareMathOperator{\argmax}{argmax}

\newcommand{\Lsig}{{\normalfont \tilde{L}}}
\newcommand{\bs}{\backslash}

\newcommand{\meta}{{\small\textsf{\textup{META}}}\xspace}
\newcommand{\base}{{\small\textsf{\textup{Base-Alg}}}}

\DeclareMathOperator{\Amaster}{\mc{A}_{\text{master}}}
\DeclareMathOperator{\tstart}{\it{t}_{{\normalfont \text{start}}}}

\newcommand{\remove}[1]{}

\newtheorem{thm}{Theorem}
\newtheorem{prop}{Proposition}
\newtheorem{cor}{Corollary}

\newtheorem{defn}{Definition}

\newtheorem{rmk}{Remark}
\let\oldrmk\rmk
\renewcommand{\rmk}{\oldrmk\em}

\SetCommentSty{mycommfont}

\usepackage{etoolbox}  
\makeatletter
\patchcmd{\algorithmic}{\addtolength{\ALC@tlm}{\leftmargin} }{\addtolength{\ALC@tlm}{\leftmargin}}{}{}
\makeatother

\newcommand{\nonl}{\renewcommand{\nl}{\let\nl}}

\SetKwRepeat{Do}{do}{while}

\newcommand\numberthis{\addtocounter{equation}{1}\tag{\theequation}}

\usepackage[format=plain,
  justification=RaggedRight,
  singlelinecheck=false,
  figurename=Fig.,
  aboveskip=7pt,
  belowskip=0pt]{caption}

\usepackage{wrapfig}
\usepackage[calc]{adjustbox}
\usepackage{cutwin}
\newlength{\strutheight}
\hypersetup{
	colorlinks,
	citecolor=blue,
	filecolor=cyan,
	linkcolor=magenta,
	urlcolor=magenta
}

\newcommand{\bld}{\textbf}
\newcommand{\ds}{\displaystyle}
\newcommand{\mc}{\mathcal}
\newcommand{\mb}{\mathbb}
\newcommand{\abs}[1]{\left| #1 \right |}

\begin{document}
\SetEndCharOfAlgoLine{}
\maketitle

\begin{abstract}
In \emph{bandit with distribution shifts}, one aims to automatically adapt to unknown changes in reward distribution, and \emph{restart} exploration when necessary. While this problem has been studied for many years, a recent breakthrough of \cite{auer2018,auer2019} provides the first adaptive procedure to guarantee an optimal (dynamic) regret $\sqrt{LT}$, for $T$ rounds, and an unknown number $L$ of changes. However, while this rate is tight in the worst case, it remained open whether faster rates are possible, without prior knowledge, if few changes in distribution are actually \emph{severe}. 

To resolve this question, we propose a new notion of \emph{significant shift}, which only counts very severe changes that clearly necessitate a restart: roughly, these are changes involving not only best arm switches, but also involving large aggregate differences in reward overtime. Thus, our resulting procedure adaptively achieves rates  always faster (sometimes significantly) than $O(\sqrt{ST})$, where $S\ll L$ only counts best arm switches, while at the same time, always faster than the optimal $O(V^{\frac{1}{3}}T^{\frac{2}{3}})$ when expressed in terms of \emph{total variation} $V$ (which aggregates differences overtime). Our results are expressed in enough generality to also capture non-stochastic adversarial settings.

\end{abstract}


\section{Introduction}



In Multi-armed bandit (MAB) an agent sequentially chooses an action, out of a finite set of $K$ actions (or \emph{arms}), based on partial and uncertain feedback in the form of rewards $Y_t(a)$ for past actions $a \in [K]$ (a \emph{pull} of arm $a$) \citep[see][for surveys]{bubeck2012a,slivkinsbook,lattimore}. The goal is to maximize the cumulative reward.

We consider the setting of {\em switching bandits}, where the distributions of arms' rewards change an unknown number of times, say $L$ times, till a time horizon $T$. Performance is then measured by a {\em dynamic regret} which compares rewards against those of the best arm at each round $t$ (i.e., the arm maximizing mean rewards $\mu_t(a) \doteq \mathbb{E}\ Y_t(a)$ over $a\in [K]$). \citet{garivier2011} showed that existing procedures \citep{auer2002nonstochastic,kocsis2006} could achieve a dynamic regret $\tilde O(\sqrt{LT})$, however, requiring knowledge of $L$. This requirement on knowledge of $L$, although impractical, remained standing till a recent breakthrough of \citet{auer2018,auer2019}, with important follow-ups in \citet{chen2019, wei2021} for contextual settings. Soon afterwards, other parameter-free approaches were studied in \citet{besson2019,mukherjee2019}. This work is concerned with the achievability of faster rates that better account for mild to severe changes.

In particular, the $\sqrt{LT}$ rate is understood to be tight only in the worst-case, since $L$ counts \emph{all} changes in distribution, even \emph{mild} ones. This is evidenced, e.g., by alternative bounds of the form $V^{\frac{1}{3}}T^{\frac{2}{3}}$ \citep[see e.g.][]{chen2019}, when the \emph{total variation} $V \doteq \sum_t \max_a\abs{\mu_{t+1}(a) - \mu_t(a)}$ (aggregating differences) is small, e.g., $V\lesssim 1/\sqrt{T}$, yielding regret $\sqrt{T}$, irrespective of the $L$ changes. However, the total variation perspective can also be pessimistic as $V$ could be large, up to $O(T)$ (yielding regret $O(T)$), even {when the best arm remains fixed across the $L$ changes}.

Outside of best arm switches, a dynamic regret of $\sqrt{T}$ remains possible \citep{allesiardo2017}, irrespective of $V$ and $L$. In light of this, achieving faster rates of $\sqrt{ST}$, for an unknown number $S\ll L$ of best arm switches, is understood as a key open problem \footnote{The open problem in \citep{foster2020}, of adapting to \emph{switching regret}, is even more general than that resolved here.} \citep{auer2019,foster2020}.

We in fact aim beyond solving this problem, as even the quantity $S$ overcounts the severity of changes, e.g., when $V$ remains small across best arm switches.
Instead we propose a new notion of \emph{significant shift} which aims to only count severe changes in mean reward that clearly necessitate a restart. Integral to our definition is the idea that a restart (in exploration) is only necessary when there are no \emph{safe} arms left to play. Here, an arm $a$ is considered safe if its total regret over any interval $[s_1, s_2]$ within a phase, namely $\sum_{t=s_1}^{s_2}  \max_{a'}\mu_t(a') - \mu_t(a)$, is $o(\sqrt{s_2 - s_1})$.


As such, a \emph{significant phase} ends only when no safe arm remains, which in particular does not count best arm switches that do not last long enough to hurt regret. For example, a change from arm $a$ to $a'$ as a new best arm, with a large gap $\mu_t(a') - \mu_t(a) = \delta$, constant over the next rounds, is ignored if another switch within the next $O(\delta^{-2})$ rounds, reverts back to $a$ as the best arm. Also, unlike with total variation, aggregate differences are ignored outside of best arm switches.

As a simple sanity check that other distributional changes beyond significant shifts are safely ignored, we show in {Proposition~\ref{prop:sanity}} that a simple oracle that effectively ignores all other changes, achieves regret $\sqrt{\abs{I}}$ over any significant phase $I$ of length $\abs{I}$. Furthermore, such a basic guarantee of $\sum_I \sqrt{\abs{I}}$ total regret over significant phases $I$, immediately implies (1) a rate of $\sqrt{\Lsig T}$ where $\Lsig \leq S\leq L$ only counts significant shifts, and (2) a rate of $V^{\frac{1}{3}}T^{\frac{2}{3}}$ in terms of total variation.

Finally we show that these guarantees are attained adaptively, i.e., with no prior knowledge of the environment {(Theorem~\ref{thm})}. Our key algorithmic innovation over previous works is that, rather than aiming to detect changes in the mean rewards $\mu_t(a)$ of arms $a$, we focus on detecting changes in the \emph{aggregate gaps in mean rewards} between arms, i.e., in $\sum_t \mu_t (a') - \mu_t (a)$, a stable quantity across milder shifts. While this quantity does not directly track the dynamic regret of arm $a$ (as the comparator $a'$ here is independent of $t$), it will turn out sufficient (see beginning of Section \ref{sec:alg}). Now, this quantity is estimable via a simple unbiased (importance-weighted) statistic, which concentrates via martingale inequalities and can thus be used for change-point detection; however, the estimate does require \emph{replays} of previously discarded arms $a'$. These replays are carefully scheduled so as to not overly affect regret, as inspired by previous work \citep[e.g.][]{auer2019, chen2019}. The resulting procedure \meta is described in Algorithm~\ref{meta-alg}.



Our results not only concern the stochastic switching bandits setting, but extend to the adversarial setting, since we allow for shifts at every round, though they may not trigger \emph{significant shifts}, and as well as allow for deterministic rewards (i.e., degenerate distributions). Formally, our setting admits a randomized, oblivious adversary which decides, a priori, an arbitrary sequence of distributions for the rewards \citep{slivkinsbook}. 

{Finally, we remark that a recent paper of \citet{abbasi-yadkori2022} independently and concurrently also resolves the problem of obtaining the $\sqrt{ST}$ rate adaptively.}

\subsection{Further Discussion on Related Work}\label{subsec:related}

The work of \citet{manegueu2021} is closest in spirit to this paper, also establishing dynamic regret bounds scaling with a generalized notion of shifts. However, their notion is considerably weaker, as it for instance counts changes resulting in large differences $|\mu_t(a)-\mu_{t-1}(a)|$ in mean rewards of a best arm $a$, even when the best arm does not change. More importantly, their procedures are non-adaptive, requiring knowledge of the number of such changes.

Many works have considered the problem of achieving regret depending only on the number of best arm switches $S$, which we discuss next.

First, in the setting of adversarial bandits with deterministic rewards, the EXP3.S procedure of \citep{auer2002nonstochastic,auer2002} can achieve a near-optimal\footnote{It remains unclear whether $\log$ terms are avoidable for adaptive procedures in this setting.} dynamic regret bound $\tilde{O}(\sqrt{ST})$, but requiring knowledge of $S$. However, for more practical situations where $S$ is unknown, the best regret guarantee to our knowledge is $O(\sqrt{ST}+T^{3/4})$, obtained by combining the Bandits-over-Bandits strategy of \citet{cheung2019learning} with EXP3.S \citep{foster2020,auer2019}. {It is also known that EXP3.S alone can obtain the suboptimal rate of $\tilde{O}(S\sqrt{T})$ for unknown $S$ by setting the parameters of the algorithm independently of $S$ \citep{auer2002}}.

The more general problem of obtaining adaptive {\em switching regret} $\sqrt{ST}$ for all $S$ remains open and is beyond the scope of this paper \citep{foster2020}. {Relatedly, recent work of \citet{zimmert21} showed the impossibility of obtaining such adaptive switching regret bounds against an adaptive adversary (where switch points may depend on the algorithm's actions).}

Second, the work of \citet{allesiardo2017} studies dynamic regret in the same randomized adversarial bandit setting as this paper (which as stated above also recovers the deterministic setting). They provide two algorithms that get regret $S\sqrt{T}$ and $T^{2/3}\sqrt{S}$, which do not match the optimal regret of $\sqrt{ST}$. Furthermore, their procedures either require knowledge of $S$ or lower-bounds on the magnitude of the changes to allow for easier detection. 

In contrast to the above works, we impose no requirements on detectability, nor knowledge of the environment, while achieving rates (at times considerably) faster than $\sqrt{ST}$.


A different body of literature considers more structured changes in reward distributions. In {\em rested rotting bandits}, the reward of an arm decreases depending on its amount of play \citep{seznec2020,levine2017,heidari2016,seznec2019}. \citet{slivkins2008} study a setting where the rewards follow a Brownian motion across time. Several works also studied a subcase of the above mentioned drifting environment: slowly varying bandits, parametrized by a local limit on the change of rewards between consecutive rounds \citep{wei-srivastava,krishnamurthy}. 
Generally, such stronger structural conditions can yield faster rates at times, often of the form $\phi\cdot \log(T)$, for some problem dependent quantity $\phi$.

Finally, various works combine adversarial and stochastic approaches to simultaneously address both settings \citep{bubeck2012b,seldin2014,auer2016}. However, these approaches measure regret to the best arm in hindsight, whereas our results hold for the stronger dynamic regret.

\section{Problem Setting}\label{sec:setup}

In our environment, an oblivious adversary decides a sequence of distributions on the rewards. This subsumes the stochastic switching bandit problem and is also a generalization of the adversarial bandit problem with deterministic rewards since our reward distributions can have zero variance.

We assume a decision space $[K]$ of $K$ arms with bounded reward distributions: arm $a$ at round $t$ has reward $Y_t(a) \in [0,1]$ with mean $\mu_t(a)$. A (possibly randomized) policy $\pi$ selects at each round $t$ some arm $\pi_t \in [K]$ and observes reward $Y_t(\pi_t)$. The goal is to minimize the {\em dynamic regret}, i.e., the expected regret to the best arm at each round. This is defined as:
\[
	R(\pi,T) \doteq \sum_{t=1}^T \max_{a\in [K]} \mu_t(a) - \mathbb{E}\left[ \sum_{t=1}^T \mu_t(\pi_t)\right].
\]

In this paper, we rely heavily on analyzing the gaps in mean rewards between arms. Therefore, for convenience, let $\delta_t(a',a) \doteq \mu_t(a') - \mu_t(a)$ denote the {\em relative gap} of arms $a$ to $a'$ at round $t$. Define the {\em worst gap} of arm $a$ as $\delta_t(a) \doteq \max_{a'\in[K]} \delta_t(a',a)$, corresponding to the instantaneous regret of playing $a$ at round $t$. Notice that the above regret is then given as $\sum_{t\in [T]} \mb{E}[\delta_t(\pi_t)]$.

\paragraph{Significant Shifts and Phases.}
First, we say arm $a$ incurs \bld{significant regret} on interval $[s_1,s_2]$ if:
\begin{equation}\label{eq:bad-arm}
	\sum_{t = s_1}^{s_2} \delta_{t}(a) \geq \sqrt{K\cdot (s_2-s_1)}\tag{$\star$}.
\end{equation}
Intuitively, such an arm is no longer safe to play. Now, if instead \eqref{eq:bad-arm} holds for no interval in a time period, the arm $a$ incurs little regret over that period.
We therefore propose to record a \emph{significant shift} only when there is no safe arm left to play. This idea leads to the following recursive definition.


\begin{defn}\label{defn:sig-shift}
	Let $\tau_0=1$. Then, {recursively for $i \geq 0$}, the $(i+1)$-th {\bf significant shift} {is recorded at time} $\tau_{i+1}$, {which denotes} the earliest time $\tau \in (\tau_i, T]$ such that \emph{for every arm} $a\in [K]$, there exists rounds $s_1<s_2, [s_1,s_2] \subseteq [\tau_i,\tau]$, such that {arm} $a$ has {significant} regret \eqref{eq:bad-arm} on $[s_1,s_2]$.

	{We will refer to intervals $[\tau_i, \tau_{i+1}), i\geq 0,$ as {\bf (significant) phases}. The unknown number of such phases (by time $T$) is denoted $\Lsig +1$, whereby $[\tau_\Lsig, \tau_{\Lsig +1})$, for $\tau_{\Lsig +1} \doteq T+1,$ denotes the last phase.}
\end{defn}




It should be clear from the above that not all shifts are counted, and in fact not all best arm switches are counted, since simply having $\delta_t(a)>0$, where $a$ was previously a best arm, does not trigger a significant shift. For further intuition that a procedure need only restart exploration upon a significant shift, notice that the last arm to trigger \eqref{eq:bad-arm} in a phase $I$, only incurs regret $O(\sqrt{\abs{I}})$ over the phase. As long as there exists a safe arm to play, a small regret may be achieved if all other arms are rejected in time as they become unsafe in a phase. This intuition is verified in the next proposition.


\begin{prop}[Sanity Check]\label{prop:sanity}
    For each round $t$ belonging to phase $[\tau_i,\tau_{i+1})$, define a good arm set $\mc{G}_t$ as the set of \bld{safe} arms, i.e., arms which do not yet satisfy \eqref{eq:bad-arm} on any subinterval of $[\tau_i,t]$. Then, consider the following \bld{oracle} procedure $\pi$: at each round $t$, $\pi$ plays a random arm $a\in \mc{G}_t$ with probability $1/|\mc{G}_t|$. We then have:
	\[
		R(\pi,T) \leq \log(K)\sum_{i=0}^{\tilde{L}} \sqrt{K(\tau_{i+1}-\tau_i)}.
	\]
\end{prop}

\begin{proof}
    See Appendix~\ref{app:oracle}.
\end{proof}

It is not hard to show, essentially by Jensen's, that a rate of the form above is always faster than $O(\sqrt{\Lsig T} \land V^{\frac{1}{3}}T^{\frac{2}{3}})$. This is shown for instance in Corollaries \ref{cor:sig-shift-num} and \ref{cor:tv}.

It is therefore left to design a procedure which mimics the above oracle, i.e., that quickly detects and rejects unsafe arms, and restarts in time whenever no safe arm is left. Note that we may not be able to estimate $\sum_t \delta_t(a)$, but since there exists a relatively safe arm $a'$ in each phase, it'll be enough to ensure small regret w.r.t. such $a'$. Thus, we will need only track $\sum_t \delta_t(a', a)$.

\section{Results Overview}\label{subsec:result}

Our main result is a dynamic regret bound of similar order to Proposition~\ref{prop:sanity} {\em without knowledge of the environment, e.g., the significant shift times, or the number of significant phases}. It is stated for our algorithm \meta (Algorithm~\ref{meta-alg} of Section~\ref{sec:alg}), which, for simplicity, requires knowledge of the time horizon $T$. Knowledge of $T$ is easily removed using a {\em doubling trick} (see Proposition~\ref{prop:horizon-free}).


\begin{thm}\label{thm}
    Let $\pi$ denote the \meta procedure. Let $\{\tau_i\}_{i =0}^{\tilde L+1}$ denote the unknown significant shift times of Definition~\ref{defn:sig-shift}. We then have for some $C>0$:
    \[
	    R(\pi,T) \leq C \log^{3}(T) \sum_{i=0}^{\Lsig} \sqrt{K\cdot (\tau_{i+1} - \tau_i)}.
    \]
\end{thm}

\begin{rmk}
	The $\log$ dependence in the rate of Theorem~\ref{thm} can be improved to
	\[
		\sum_{i=0}^{\Lsig} \log(K)\log^{3/2}(KT)\sqrt{K\cdot (\tau_{i+1}-\tau_i)} + \log(K)\log^2(KT)\cdot K,
	\]
	by modifying the threshold in \eqref{eq:elim} to $\sqrt{K\log(T)\cdot(s_2-s_1)\vee K^2\log^2(T)}$ and adjusting Proposition~\ref{prop:concentration} in the same manner. This is avoided in the analysis to simplify presentation.
\end{rmk}

Note that, while the rates above are similar to the \emph{oracle} rates on Proposition~\ref{prop:sanity}, \meta may not actually achieve regret $\sqrt{\abs{I}}$ over \emph{each} phase $I$, but only guarantees $\sum_I\sqrt{\abs{I}}$ on aggregate.

The following corollary is immediate from Jensen's inequality.

\begin{cor}[Bounding by Number of Significant Shifts]\label{cor:sig-shift-num}
    Using the same notation of Theorem~\ref{thm}:
    \[
	    R(\pi,T) \leq C\log^{3}(T)\cdot \sqrt{(\Lsig + 1) K T}.
    \]
\end{cor}

Since $\Lsig \leq S$, we have by Corollary~\ref{cor:sig-shift-num} that \meta recovers the $\tilde{O}(\sqrt{SKT})$ rate of EXP3.S, however without knowledge of $\Lsig$ or $S$. The rate can in fact be much faster when $\Lsig \ll S$.



The $\tilde{O}(\sqrt{\Lsig K T})$ rate is optimal up to $\log$ terms in the worst-case since any algorithm which has to solve $\Lsig$ independent adversarial bandit problems of length $T/\Lsig$ {can be forced to suffer} regret $\Omega(\sqrt{KT/\Lsig})$ on each phase, following similar arguments as in \citet{auer2002,auer2019}. 


The next corollary asserts that Theorem~\ref{thm} also recovers the optimal rate for the {\em drifting environment} setting, i.e., in terms of total-variation $V$. The proof (Appendix~\ref{app:tv-proof}) follows from the definition of significant shift (Definition~\ref{defn:sig-shift}) in that the total-variation over every phase $[\tau_i,\tau_{i+1})$ can be bounded below by roughly $1/\sqrt{\tau_{i+1} - \tau_i}$.

\begin{cor}[Bounding by Total Variation]\label{cor:tv}
    Let $V \doteq \sum_{t=2}^T \max_{a\in [K]} |\mu_t(a) - \mu_{t-1}(a)|$ be the unknown total variation of change in the rewards. Using the same notation of Theorem~\ref{thm}:
    \[
	    R(\pi,T) \leq C\log^{3}(T) \left(\sqrt{KT} + (KV)^{1/3} T^{2/3}\right).
    \]
\end{cor}

Finally, as stated above, a simple doubling trick removes the need to know $T$.

\begin{prop}\label{prop:horizon-free}
	The doubling-horizon version of \meta has dynamic regret
	\[
		\tilde{O}\!\left(\sqrt{(\Lsig+1) K T} \land \left(\sqrt{KT} + (KV)^{1/3} T^{2/3})\right)\right).
	\]
\end{prop}




\section{Algorithm}\label{sec:alg}

\begin{algorithm2e}[h] \small 
\caption{{\bld{M}eta-\bld{E}limination while \bld{T}racking \bld{A}rms (\meta)}}
\label{meta-alg}
	{\nonl \bld{Input:} horizon $T$.}\\
	  \bld{Initialize:} round count $t \leftarrow 1$.\\
	  \textbf{Episode Initialization (setting global variables {\normalfont $t_{\ell},\Amaster,B_{s,m}$})}:\\
	  \Indp $t_{\ell} \leftarrow t$. \label{line:ep-start}  \tcp*{$t_{\ell}$ indicates start of $\ell$-th episode.}
	  $\Amaster \leftarrow [K]$ \label{line:define-end} \tcp*{Master candidate arm set.}
    For each $m=2,4,\ldots,2^{\lceil\log(T)\rceil}$ and $s=t_{\ell}+1,\ldots,T$:\\
        \Indp Sample and store $B_{s,m} \sim \text{Bernoulli}\left(\frac{1}{\sqrt{m\cdot (s-t_{\ell})}}\right)$. \label{line:add-replay} \tcp*{Set replay schedule.}
        \Indm
	\Indm
	 \vspace{0.2cm}
	 Run $\base(t_{\ell},T + 1 - t_{\ell})$. \label{line:ongoing-base} \\
  \lIf{$t < T$}{restart from Line 2 (i.e. start a new episode).
  \label{line:restart}}
\end{algorithm2e}


 \begin{algorithm2e}[h]\small
 	\caption{{\base$(\tstart,m_0)$: elimination with randomized arm-pulls}}
 \label{base-alg}
 {\nonl \textbf{Input}: starting round $\tstart$, scheduled duration $m_0$.}\\
 \textbf{Initialize}: $t \leftarrow \tstart$, $\mc{A}_t \leftarrow [K]$. \tcp*{$t$ and $\mc{A}_t$ are global variables.}
 	  \While{$t \leq T$}{
 		  Play a random arm $a\in \mc{A}_t$ selected with probability $1/|\mc{A}_t|$. \label{line:play}\\
		  Let $\mc{A}_{\text{current}} \leftarrow \mc{A}_{t}$ \label{line:current-arm-set} \tcp*{Save current candidate arm set.}
		  Increment $t \leftarrow t+1$.\\
 		  \uIf{$\exists m\text{{\normalfont\,such that }} B_{t,m}>0$}{
                 Let $m \doteq \max\{m \in \{2,4,\ldots,2^{\lceil \log(T)\rceil}\}:B_{t,m}>0\}$. \tcp*{Set maximum replay length.}
                 Run $\base(t,m)$.\label{line:replay} \tcp*{Replay interrupted by child replay.}
 		   }
		   \lIf{$t > \tstart + m_0$}{
		   RETURN.
			}
 		 \bld{Evict bad arms:}\\
		 \Indp $\Amaster \leftarrow \Amaster \bs \{a\in[K]:\text{$\exists $ rounds $[s_1,s_2]\subseteq [t_{\ell},t)$ s.t. \eqref{eq:elim} holds}\}$. \label{line:evict-master}\\
		 $\mc{A}_{t} \leftarrow \mc{A}_{\text{current}} \bs \{a\in [K]:\text{$\exists $ rounds $[s_1,s_2]\subseteq [\tstart,t)$ s.t. \eqref{eq:elim} holds}\}$. \label{line:evict-At}\\
 		\Indm
		\bld{Restart criterion:} \lIf{$\normalfont\Amaster=\emptyset$}{RETURN.}
 	}
\end{algorithm2e}

As explained earlier in the introduction, a main departure from past work is that, in order to detect a (significant) shift, we aim to detect changes in aggregate gaps $\sum_{t=s_1}^{s_2} \delta_t(a', a)$ between pairs of arms, over intervals $[s_1, s_2]$, rather than, e.g., changes in mean rewards $\mu_t(a)$. In particular this allows us to avoid triggering a restart upon benign changes in distribution.

However, as previously discussed, $\sum_{t=s_1}^{s_2} \delta_t(a', a)$ does not directly track the dynamic regret $\sum_{t=s_1}^{s_2} \max_{a'} \delta_t(a', a)$. Fortunately, for one, it is a lower bound on the dynamic regret; thus if it is large, so is the dynamic regret. On the other hand, if $\max_{a'} \sum_{t=s_1}^{s_2} \delta_t(a', a)$ remains small for some $a$ over any $[s_1, s_2]$, we can rely on the very definition of \emph{significant shift} to infer that $a$ remains relatively safe to play: let $a'$ be the last arm to incur \emph{significant regret} \eqref{eq:bad-arm} in a phase $I$, $a$ and $a'$ must then have similar regret $\sqrt{\abs{I}}$ over $I$.

Now, as we eliminate arms overtime, estimating $\sum_{t=s_1}^{s_2} \delta_t(a', a)$ involves \emph{replaying} previously eliminated arms. This has to be done on a careful schedule so as not to overly affect regret when there is no change; here we simply borrow from similar schedules as in previous work \citep{auer2019,chen2019,wei2021}.

We next establish some useful terminology for discussing our approach in more detail.

\paragraph{Terminology and Overview.} Our main procedure \meta operates in {\em episodes}, starting each episode by playing a \emph{base algorithm} for a possible duration equal to the rounds left till $T$. Now, a base algorithm occasionally {\em activates} its own base algorithms of varying durations (Line~\ref{line:replay} of Algorithm~\ref{base-alg}), called {\em replays}, aimed at detecting changes according to the aforementioned schedule (stored in the variable $\{B_{s,m}\}$). We refer to the base algorithm playing at round $t$ as the {\em active base algorithm}. This results in recursive calls, from \emph{parent} to \emph{child} instances of \base, as depicted in Figure~\ref{fig:replays}.

\begin{figure}

\centering
\includegraphics[scale=0.5]{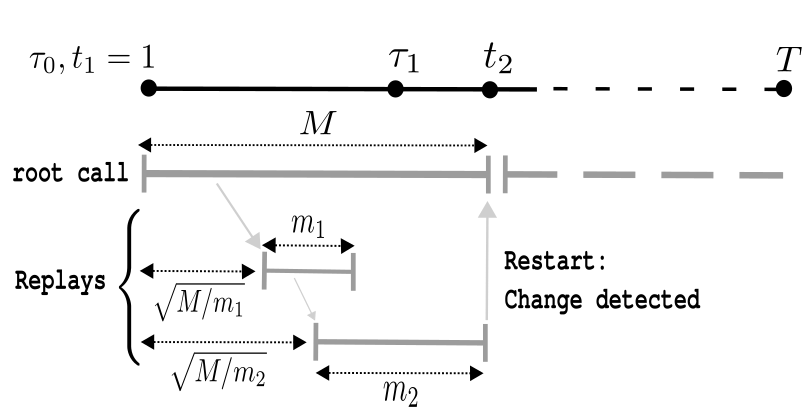}
\caption{\noindent \small
Shown are two replay durations $m= m_1\text{ or } m_2$, and corresponding replays are shown (gray segments) at intervals of roughly ${\sqrt{M/m}}$ rounds, where $M$ is the eventual length of an episode. Each replay (gray segments) aims to detect a $1/\sqrt{m}$ magnitude \emph{change}, i.e., an average gap $\frac{1}{m}\sum_{t=1}^m \delta_t(a)$ of order $1/\sqrt{m}$. As a recursive procedure, the replays of \base\, form a \emph{parent-child} relationship as depicted.
}
\label{fig:replays}
\end{figure}

Each instance of \base\ maintains at each round $t$ a candidate arm set $\mc{A}_t$, initialized to $[K]$ at $t=\tstart$ and further refined as it eliminates suboptimal arms.
Instances of \base\ and \meta share information, in the form of
{\em global variables} as listed below:
\begin{itemize}[leftmargin=*]
\item All variables defined in \meta, namely  $t_{\ell},t,\mc{A}_{\text{master}},\{B_{s,m}\}$ (see Lines~\ref{line:ep-start}--\ref{line:add-replay} of Algorithm~\ref{meta-alg}).
\item All arms played at any round $t$, along with observed rewards $Y_t(a)$, and the candidate arm set $\mc{A}_t$. As such, the size of $\mc{A}_t$ can fluctuate throughout an episode as it is reset to $[K]$ by new replays, i.e., as a base algorithm is activated by a parent, and again reset as the parent resumes (as maintained via $\mc{A}_{\text{current}}$).
\end{itemize}

          By sharing these global variables, any replay can trigger a new episode: every time an arm is evicted by a replay, it is also evicted from $\Amaster$, essentially the candidate arm set for the current episode. A new episode is triggered when $\Amaster$ is empty, i.e., there is no \emph{safe} arm left to play.



\paragraph{Eviction Criteria and Critical Estimates.} A running intuition so far is that  \base\ ejects any arm $a$ from $\mc{A}_t$ when it determines that $\sum_{t=s_1}^{s_2}\delta_t(a', a)$ is large. Note that any rejected arm is also immediately rejected by the parent when a child terminates.


Now, the quantity $\sum_{t=s_1}^{s_2} \delta_t(a', a)$ is estimated as
$\sum_{t=s_1}^{s_2} \hat \delta_t(a', a)$, whereby the relative gap $\delta_t(a',a)$ is estimated by importance weighting as:
\begin{align}
	\hat{\delta}_t(a',a) \doteq |\mc{A}_{t}|\cdot \left( Y_t(a')\cdot \pmb{1}\{\pi_{t}=a'\} - Y_t(a)\cdot \pmb{1}\{\pi_{t}=a\}\right).
	\label{eq:delta-t-estimate}
\end{align}

Note that the above is an unbiased estimate of $\delta_t(a', a)$ whenever $a'$ and $a$ are both in $\mc{A}_t$ at time $t$. It then follows that the difference $\sum_{t=s_1}^{s_2} \left (\hat \delta_t (a', a) - \delta_t(a', a)\right)$ is a martingale that concentrates at a rate roughly $\sqrt{K(s_2 -s_1)}$ (see Proposition \ref{prop:concentration}).


An arm $a$ is then evicted at round $t$ if, for some fixed $C_0 > 0$ \footnote{$C_0 > 0$ needs to be sufficiently large, but does not depend on the horizon $T$ or any distributional parameters. An appropriate value can be derived from the regret analysis.}, $\exists$ rounds $s_1 < s_2\leq t$ such that:
\begin{equation}\label{eq:elim}
	\max_{a'\in [K]}	\sum_{t=s_1}^{s_2} \hat{\delta}_s(a',a) > \log(T)\sqrt{C_0\cdot \left( K (s_2-s_1) \vee K^2\right)}.
\end{equation}

\section{Regret Analysis}\label{sec:overview}

Continuing from the discussion at the beginning of Section~\ref{subsec:result}, for any arm $a$ being played by \meta, ideally, the regret to the last arm $a'$ to incur significant regret in any given phase remains small, lest $a$ would be evicted before such regret $\sum_{t=s_1}^{s_2} \delta_t(a',a)$ is too large over a given interval $[s_1,s_2]$. However, we can only reliably estimate such regret via $\sum_{t=s_1}^{s_2} \hat{\delta}_t(a',a) $ when both $a, a'$ are being played over $[s_1,s_2]$ (in fact, this estimator is only unbiased when this is the case).

To this end, we will decompose each phase into \emph{bad segments} where $\sum_{t=s_1}^{s_2} \delta_t(a',a)$ is large (see Definition~\ref{defn:bad-segment}); we then argue that, while it is safe to miss a few such bad segments, a timely replay of $\base$ is likely to occur over some such segment, ensuring both $a, a'$ are being played, and thus leading to the detection that $a$ is no longer safe.
Note that it is easy to argue that such a perfect replay will not evict arm $a'$ since otherwise $\max_{a''} \sum_{t=s_1}^{s_2} \delta_t(a'',a') \leq \sum_{t=s_1}^{s_2} \delta_t(a')$ must be large, i.e., $a'$ must have significant regret on some interval before the end of the phase. For simplicity, we'll refer to such well-timed replays as {\em perfect replays}. This is discussed in more detail in Section \ref{subsec:episode-regret}.

\paragraph{Preliminaries.} Throughout the proof $c_1,c_2,\ldots$ will denote positive constants not depending on $T$ or any distributional parameters. Recall from Line~\ref{line:ep-start} of Algorithm~\ref{meta-alg} that $t_{\ell}$ is the first round of the $\ell$-th episode. WLOG, there are $T$ total episodes and, by convention, we let $t_{\ell} \doteq T+1$ if only $\ell-1$ episodes occurred by round $T$. 

We first handle the simple case of $T<K$.
In this case, the bound of Theorem~\ref{thm} vacuously holds:
\[
	\sum_{i=0}^{\tilde{L}} \sqrt{K\cdot (\tau_{i+1} - \tau_i)} \geq \sqrt{K\cdot (\tau_{\Lsig+1} - \tau_0)} = \sqrt{K\cdot (T+1-1)} > T.
\]
Thus, it remains to show Theorem~\ref{thm} for $T \geq K$.

\subsection{Decomposing the Total Regret}

Continuing from the earlier discussion, we first decompose the total regret into (i) the regret of the last arm to incur significant regret in a phase and (ii) the relative regret of \meta to this {\em last safe arm}. We start by defining notation, to be used throughout, for this last safe arm,

\begin{defn}[Last Safe Arm in a Phase]
	As described in Section~\ref{sec:setup}, there is a \bld{safe} arm to play in each phase $i$, namely the last arm to incur significant regret \eqref{eq:bad-arm} in phase $i$, which we denote $a_i^\sharp$. We then let $a_t^\sharp$, $t \in [0, T]$, track $a_i^\sharp$ across phases, that is $a_t^\sharp \doteq a_i^\sharp$ for $t\in [\tau_i,\tau_{i+1})$.
\end{defn}

Then, the total regret decomposes as:
\[
	\mb{E}\left[ \sum_{t=1}^T \delta_t(\pi_t) \right] = \sum_{t=1}^T \delta_t(a_t^\sharp) + \mb{E}\left[ \sum_{t=1}^T \delta_t(a_t^\sharp,\pi_t) \right].
\]
Note that there is no randomness in the first sum of the R.H.S. above since the safe arm $a_t^\sharp$ is non-random for any fixed round $t$. This first sum is handled similarly to the guarantee for the oracle procedure of Proposition~\ref{prop:sanity}, as each arm $a_i^\sharp$ is safe to play in phase $i$. In particular, showing this sum is of the desired order follows from taking $\mc{S}_t=\{a_t^\sharp\}$ in Lemma~\ref{lem:combine} of Appendix~\ref{app:oracle}.


So, it remains to bound the relative regret of $\pi_t$ to the safe arm $a_t^\sharp$. To this end, we will first show that the estimation error of the aggregate relative gap over an interval is suitably small.

\subsection{Estimating the Aggregate Gap over an Interval}

We first recall a Bernstein-type martingale tail bound, namely, Freedman's inequality, which yields as tight a concentration rate as necessary for our purpose (see Proposition \ref{prop:concentration}):

\begin{lemma}[Theorem 1 of \citet{beygelzimer2011}]\label{lem:martingale-concentration}
	Let $X_1,\ldots,X_n\in\mb{R}$ be a martingale difference sequence with respect to some filtration $\{\mc{F}_0,\mc{F}_1,\ldots\}$. Assume for all $t$ that $X_t\leq R$ a.s. and that $\sum_{i=1}^n \mb{E}[X_i^2|\mc{F}_{i-1}] \leq V_n$ a.s. for some constant  $V_n$ only depending on $n$.
	Then for any $\delta\in (0,1)$ and $\lambda\in [0,1/R]$, with probability at least $1-\delta$, we have:
	\[
		\sum_{i=1}^n X_i \leq (e-2) \lambda V_n + \frac{\log(1/\delta)}{\lambda}.
	\]
\end{lemma}

Letting $\lambda = \frac{1}{R} \land \sqrt{\frac{\log(1/\delta)}{V_n}}$ in Lemma~\ref{lem:martingale-concentration} implies that if $V_n \geq R^2\log(1/\delta)$, then the above L.H.S. is upper-bounded by ${(e-1)}\sqrt{V_n \log(1/\delta)}$. On the other hand, if $V_n < R^2\log(1/\delta)$, then the bound becomes $(e-1) R \log(1/\delta)$. Thus, in any case, Lemma~\ref{lem:martingale-concentration} gives us that with probability at least $1-\delta$:
\begin{equation}\label{eq:concentration-optimized}
	\sum_{i=1}^n X_i \leq (e-1)\left( \sqrt{V_n\log(1/\delta)} + R\log(1/\delta)\right).
\end{equation}

We next apply Lemma~\ref{lem:martingale-concentration} for the result below, where the estimate $\hat{\delta}_t(a',a)$ is given in  \eqref{eq:delta-t-estimate}.

\begin{prop}\label{prop:concentration}
    Let $\mc{E}_1$ be the event that for all rounds $s_1<s_2$ and all arms $a,a'\in [K]$:
	\begin{equation}\label{eq:error-bound}
		\left| \sum_{t=s_1}^{s_2} \hat{\delta}_t(a',a) - \sum_{t=s_1}^{s_2} \mb{E}\left[\hat{\delta}_t(a',a)\mid \mc{F}_{t-1}\right] \right| \leq c_1\log(T) \left( \sqrt{K(s_2-s_1)} + K\right),
	\end{equation}
    for an appropriately large constant $c_1$, and where $\mc{F} \doteq \{\mc{F}_t\}_{t=1}^T$ is the filtration with $\mc{F}_t$ generated by $\{\pi_s,Y_s(\pi_s)\}_{s=1}^t$. Then, $\mc{E}_1$ occurs with probability at least $1-1/T^2$. 
\end{prop}

\begin{proof}
	The martingale difference $\hat{\delta}_t(a',a) - \mb{E}[ \hat{\delta}_t(a',a) | \mc{F}_{t-1}]$ is clearly bounded above by $2K$ for all rounds $s$ and all arms $a,a'$. We also have a cumulative variance bound:
	\begin{align*}
		\sum_{t=s_1}^{s_2}	\mb{E}\ \left [\hat{\delta}_t^2(a',a) \mid \mc{F}_{t-1}\right] 
												  &\leq \sum_{t=s_1}^{s_2} |\mc{A}_t|^2 \cdot \mb{E}\ \left [\pmb{1}\{ \pi_t=a\text{ or }a'\}\mid \mc{F}_{t-1}\right ]\\
												  &\leq \sum_{t=s_1}^{s_2} |\mc{A}_t|^2 \cdot \frac{2}{|\mc{A}_t|}=  \sum_{t=s_1}^{s_2} 2|\mc{A}_t|
												  \leq 2K\cdot (s_2-s_1).
	\end{align*}
	Then, the result follows from \eqref{eq:concentration-optimized}, and taking union bounds over arms $a,a'$ and $s_1,s_2$.
\end{proof}


\subsection{Relating Episodes to Significant Phases}

We next show that w.h.p. a restart occurs (i.e., a new episode begins) only if a significant shift has occurred sometime within the episode. {As an important consequence, each significant phase straddles at most two episodes; this fact comes in handy in summing the regret over episodes (Section~\ref{subsec:sum-regret}).}

Recall from Definition~\ref{defn:sig-shift} that $\tau_1,\tau_2,\ldots,\tau_{\Lsig}$ are the times of the significant shifts and that $t_1,\ldots,t_{T}$ are the episode start times. 

\begin{lemma}[{\bf Restart Implies Significant Shift}]\label{lem:counting-eps}
	On event $\mc{E}_1$, for each episode $[t_{\ell},t_{\ell+1})$ with $t_{\ell+1}\leq T$ (i.e., an episode which concludes with a restart), there exists a significant shift $\tau_i\in [t_{\ell},t_{\ell+1})$.
\end{lemma}

\begin{proof}
	Fix an episode $[t_{\ell},t_{\ell+1})$. Then, by Line~\ref{line:evict-master} of Algorithm~\ref{meta-alg}, every arm $a\in [K]$ was evicted from $\mc{A}_{\text{master}}$ at some round $s\in [t_{\ell},t_{\ell+1})$, i.e. \eqref{eq:elim} is true for some interval $[s_1,s_2]\subseteq [t_{\ell},s)$. It suffices to show that this implies arm $a$ incurs significant regret \eqref{eq:bad-arm} on $[s_1,s_2]$.

	Suppose \eqref{eq:elim} triggers the eviction of arm $a$. By \eqref{eq:error-bound} and \eqref{eq:elim}, we have that there is an arm $a'\neq a$ such that (using the notation of Proposition~\ref{prop:concentration}) for large enough $C_0>0$ and some $c_2>0$:
	\begin{equation}\label{eq:large-aggregate-gap}
		\sum_{t=s_1}^{s_2} \mb{E}\left[ \hat{\delta}_t(a',a) \mid \mc{F}_{t-1}\right] \geq c_2 \log(T) \sqrt{K(s_2-s_1) \vee K^2}.
	\end{equation}
	Next, note that by comparing Lines~\ref{line:evict-master} and \ref{line:evict-At}, we see that any arm which is evicted from $\mc{A}_t$ at any round $t \in [t_{\ell},s)$ must have also been evicted from $\mc{A}_{\text{master}}$ in the same round. Thus, if arm $a$ is evicted from $\mc{A}_{\text{master}}$ at round $s$, then we must have that $a\in \mc{A}_t$ for all $t\in [t_{\ell},s)$. Then
	\[
	\mb{E}[  \hat{\delta}_t(a',a) \mid \mc{F}_{t-1} ] = \begin{cases}
			\delta_t(a',a) & a'\in \mc{A}_t\\
			-\mu_t(a) & a'\not\in\mc{A}_t
		\end{cases}
	\]
	In either case, the above L.H.S. expectation is bounded above by $\delta_t(a',a) \leq \delta_t(a)$. Thus, \eqref{eq:large-aggregate-gap} implies arm $a$ incurs significant regret \eqref{eq:bad-arm} on $[s_1,s_2]$ since the L.H.S. is upper-bounded by $\sum_{t=s_1}^{s_2} \delta_t(a)$.
\end{proof}

\subsection{Bounding the Regret Within Each Episode}\label{subsec:episode-regret}

The relative regret of \meta to the last safe arm within each episode $[t_{\ell},t_{\ell+1})$ will be given in terms of the phases it intersects. To this end, we introduce the following new notation.

\begin{defn}
	Let $\textsc{Phases}(t_{\ell},t_{\ell+1}) \doteq \{i\in [\Lsig]:[\tau_i,\tau_{i+1})\cap [t_{\ell},t_{\ell+1}) \neq \emptyset\}$, i.e., denote those phases intersecting episode $\ell$.
\end{defn}

Then, our main claim is as follows, w.r.t. the event $\mc{E}_1$ of Proposition \ref{prop:concentration}:
\begin{equation}\label{eq:regret-episode}
	\mb{E}\left[ \sum_{t=t_{\ell}}^{t_{\ell+1}-1} \delta_t(a_t^\sharp,\pi_t) \right] \leq \frac{1}{T} + c_3\log^{3}(T) \cdot \mb{E}\left[ \pmb{1}\{\mc{E}_1\} \sum_{i \in \textsc{Phases}(t_{\ell},t_{\ell+1})} \sqrt{K\cdot (\tau_{i} - \tau_{i-1})} \right].
\end{equation}
Expectations above are taken over all randomness in both algorithm and environment.
%
We will be comparing the safe arm $a_t^\sharp$ to any arm \meta considered safe throughout the episode.

\begin{defn}[Last Master Arm]
We let $a_\ell$ denote any arm surviving at time $t_{\ell +1}-1$ in $\mc{A}_{\text{master}}$.
\end{defn}

The relative regret to the last safe arm over episode $\ell$ is then decomposed into the following two quantities:
\begin{enumerate}[(a),leftmargin=*]
	\item The relative regret of $\pi_t$ w.r.t. the \emph{last master arm} $a_{\ell}$.\label{item:regret-ell}
	\item The relative regret of the last master arm $a_{\ell}$ to last safe arms $a_t^\sharp$. \label{item:regret-persistent}
\end{enumerate}
Namely, we have:
\begin{equation}
	\mb{E} \left[ \sum_{t=t_{\ell}}^{t_{\ell+1}-1} \delta_t(a_t^\sharp,\pi_t) \right] =
		\underbrace{\mb{E}\left[ \sum_{t=t_{\ell}}^{t_{\ell+1}-1} \delta_t(a_{\ell},\pi_t)\right] }_{\ref{item:regret-ell}} + \underbrace{\mb{E} \left[ \sum_{t=t_{\ell}}^{t_{\ell+1}-1} \delta_t(a_t^\sharp,a_{\ell}) \right] }_{\ref{item:regret-persistent}}.
	\label{eq:episode_decomposition}
\end{equation}
We then proceed to show that each of \ref{item:regret-ell} and \ref{item:regret-persistent} are of order \eqref{eq:regret-episode}.

An {\bf immediate difficulty} is that these two quantities involve interdependencies between the random variables $t_\ell, t_{\ell+1}, \pi_t$ and $a_\ell$, which require careful handling. Our general approach is to first condition on just $t_\ell$, whereby appropriate surrogates for these various quantities can be bounded \emph{pointwise}, i.e., independent of the values of $\pi_t, t_{\ell+1}$ and $a_\ell$. We go over more detailed intuition below.

\paragraph{$\bullet$ Bounding \ref{item:regret-ell}.}
Here we first assume the event ${\cal E}_1$ of Proposition \ref{prop:concentration}, whereby for any $a, a'$ retained together in a given interval of time $[s_1, s_2]$, i.e., $a, a' \in {\cal A}_t, \ \forall t\in [s_1, s_2]$, we have that $\sum_{t=s_1}^{s_2} \delta_t(a', a) \lesssim \sqrt{K(s_2 - s_1)}$. Now, $a_\ell$, by definition, is any $a'$ retained at all rounds in episode $\ell$, while $\delta_t(a', \pi_t) = \sum_{a \in [K]}\delta_t(a', a)\cdot \pmb{1}\{\pi_t = a\} $. We may therefore reduce the problem to bounding relative regrets $\sum_t \delta_t(a', a)\cdot \pmb{1}\{\pi_t = a\}\cdot \pmb{1}\{{\cal E}_1\}$, by carefully considering such intervals of time where $a$ is also retained. In particular, first an interval from $t_\ell$ to $t_\ell(a)$, defined as the round at which $a$ is evicted from $\mc{A}_{\text{master}}$, and then intervals within replays of $\base$ which bring $a$ back into ${\cal A}_t$. Importantly, bounding the regret within replays crucially uses the fact that sufficiently few replays are expected.
{The details are found in Appendix~\ref{subsec:2}.}

\paragraph{$\bullet$ Bounding \ref{item:regret-persistent}.}
This is most involved, a main difficulty arising from the fact that, if arm $a_t^\sharp$ is evicted from $\Amaster$ by some time $s_1$, large aggregate values of $\sum_{t= s_1}^{s_2}\delta_t(a_t^\sharp,a_{\ell})$ may go undetected outside of well-timed replays of \base.
Our main strategy, as discussed at the start of Section \ref{sec:overview} is to divide up every phase $i$ intersecting remaining rounds $[t_\ell, T]$ into {\em bad segments} where $a_{\ell}$ incurs significant regret to arm $a^\sharp_i \doteq a^\sharp_t$, and argue that few such bad segments may occur before a well-timed replay occurs that finally evicts $a_\ell$.


%


Note that such an argument is independent of the episode end time $t_{\ell+1}-1$, so \emph{we only need to condition on $t_\ell$}. A difficulty remains in that $a_{\ell}$ is undetermined till time $t_{\ell+1}$, which we circumvent by defining bad segments w.r.t. any possible arm $a$, and arguing that any such $a$ is evicted in time.




\begin{defn}\label{defn:bad-segment}
	Fix $t_{\ell}$, and let $[\tau_i,\tau_{i+1})$ be any phase intersecting $[t_{\ell},T)$. For any arm $a$, define rounds $s_{i,0}(a),s_{i,1}(a),s_{i,2}(a)\ldots\in [t_{\ell} \vee \tau_{i}, \tau_{i+1})$ recursively as follows: let $s_{i,0}(a) \doteq t_{\ell} \vee \tau_{i}$ and define $s_{i,j}(a)$ as the smallest round in $(s_{i,j-1}(a),\tau_{i+1})$ such that arm $a$ satisfies for some fixed $c_4>0$:
		\begin{equation}\label{eq:segment}
			\sum_{t = s_{i,j-1}(a)}^{s_{i,j}(a)} \delta_{t}(a_i^\sharp,a) \geq c_4 \log(T)\sqrt{K\cdot (s_{i,j}(a) - s_{i,j-1}(a))},
		\end{equation}
		if such a round $s_{i,j}(a)$ exists. Otherwise, we let the $s_{i,j}(a) \doteq \tau_{i+1}-1$. We refer to any interval $[s_{i,j-1}(a),s_{i,j}(a))$ as a {\bf critical segment}, and as a {\bf bad segment} (w.r.t. arm $a$) if \eqref{eq:segment} above holds.
\end{defn}



We can restrict attention to bad segments, since outside of this, any critical segment--in fact, at most one per phase--contributes small regret as \eqref{eq:segment} is reversed. The following proposition, which relates the concentration bound of Proposition \ref{prop:concentration} to \eqref{eq:segment}, establishes crucial guarantees on bad segments.


	\begin{prop}{(proof in Appendix~\ref{subsec:3})}\label{cor:behavior}
		Suppose event $\mc{E}_1$ holds. Let $[s_{i,j}(a),s_{i,j+1}(a))$ be a bad segment with respect to arm $a$. Fix an integer $m \geq s_{i,j+1}(a) - s_{i,j}(a)$. Then:
		\begin{enumerate}[(i)]
			\item No run of $\base(\tstart,m)$ with $\tstart\in  [s_{i,j}(a),s_{i,j+1}(a)]$ {\bf ever evicts arm $a_i^\sharp$}. \label{2a} 

			\item If $a_i^\sharp \in \bigcap_{t= \tilde s}^{s_{i,j+1}(a)}\mc{A}_t$, where $\tilde{s} \doteq \lceil\frac{s_{i,j}(a)+s_{i,j+1}(a)}{2}\rceil$, then arm $a$ is evicted by round $s_{i,j+1}(a)$. \label{2b}
		\end{enumerate}
	\end{prop}
	Note that such an eviction implies crucially an eviction from $\Amaster$. The above guarantees lead to the following notion of \emph{perfect replay}, well-timed to evict arm $a$ over a given bad segment.


	\begin{defn}\label{defn:perfect-replay}
        Given a bad segment $[s_{i,j}(a),s_{i,j+1}(a))$, a {\bf perfect replay} designates a call of \\ $\base(\tstart,m)$ where $\tstart\in [s_{i,j}(a),\tilde{s}]$ (as defined in \ref{2b}) and $m\geq s_{i,j+1}(a) - s_{i,j}(a)$
    \end{defn}

    \begin{cor}\label{cor:perfect}
	    Under the conditions of Proposition~\ref{cor:behavior}, a perfect replay will evict arm $a$ from $\Amaster$.

	\end{cor}

	All that is left is to show that, for any arm $a$ (in particular, $a=a_{\ell}$), a perfect replay is scheduled with high probability before too many bad segments w.r.t. $a$ elapse. This leads to an immediate bound on the regret of any $a$ to $a_i^\sharp$ over the phases $[\tau_i,\tau_{i+1})$ intersecting episode $[t_{\ell},t_{\ell+1})$. The full details are found in Appendix~\ref{subsec:3}.


	\subsection{Summing the regret over episodes.}\label{subsec:sum-regret}

	{

		Recall from earlier that there are WLOG $T$ total episodes with the convention that $t_{\ell} \doteq T+1$ if only $\ell$ episodes occur by round $T$. Then, summing our episode regret bound \eqref{eq:regret-episode} over $\ell$ gives: 
\begin{align*}
	\sum_{\ell=1}^T \mb{E} \left[ \sum_{t=t_{\ell}}^{t_{\ell+1}-1} \delta_t(a_t^\sharp,\pi_t) \right] &\leq 1 + c_3\log^3(T)\sum_{\ell=1}^T \mb{E}\left[ \pmb{1}\{\mc{E}_1\} \sum_{i \in \textsc{Phases}(t_{\ell},t_{\ell+1})} \sqrt{K\cdot (\tau_i-\tau_{i-1})} \right].
\end{align*}
Recall here that $\mc{E}_1$ is the good event over which the concentration bounds of Proposition~\ref{prop:concentration} hold. Then, using the fact that, on event $\mc{E}_1$, each phase $[\tau_i,\tau_{i+1})$ intersects at most two episodes (Lemma~\ref{lem:counting-eps}), the sum over $\ell\in[T]$ on the above R.H.S. becomes:
\[
	\mb{E}\left[ \pmb{1}\{\mc{E}_1\}\sum_{\ell=1}^T  \sum_{i \in \textsc{Phases}(t_{\ell},t_{\ell+1})} \sqrt{K\cdot (\tau_i - \tau_{i-1})}\right] \leq 2\sum_{i=1}^{\Lsig} \sqrt{K\cdot (\tau_{i+1}-\tau_i)}.
\]
This concludes the proof of Theorem~\ref{thm}.
}

\section{Conclusion}

We have shown that it is possible to adapt optimally to an unknown number of \emph{significant} shifts---a new notion proposed here---resulting in rates always faster than optimal total variation rates, while at the same time resolving the open problem of adaptivity to an unknown number of best arm switches $S$. Our rates can in fact be much faster than when expressed in terms of $S$ or total variation, as the notion of significant shift is considerably milder.

The more general problem of adaptive \emph{switching regret} \citep{foster2020}, where one aims to compete against any sequence of arms (as opposed to the best arm at each round), remains open.

\section*{Acknowledgements}
Samory Kpotufe thanks Google AI Princeton, and the Institute for Advanced Study at Princeton for hosting him during part of this project. He also acknowledges support from NSF:CPS:Medium:1953740 and the Alfred P. Sloan Foundation.


\bibliography{colt.bib}

\newpage
\appendix

\section{Proof of Proposition~\ref{prop:sanity}}\label{app:oracle}

We show a slightly more general version of Proposition~\ref{prop:sanity} below which comes in handy in the proof of Theorem~\ref{thm}.

	\begin{lemma}\label{lem:combine}
		Let $\{\mc{S}_t\}_{t=1}^T$ be a fixed sequence of arm-sets such that $\mc{S}_t \subseteq \mc{G}_t$ and, for any fixed phase $[\tau_i,\tau_{i+1})$, $\mc{S}_t \supseteq \mc{S}_{t+1}$ for all $t\in [\tau_i,\tau_{i+1})$. Let $\pi$ be a procedure which, at each round $t$, plays an arm uniformly at random from $\mc{S}_t$. We then have
	    \[
		    R(\pi,T) \leq \log(K)\sum_{i=0}^{\Lsig} \sqrt{K\cdot (\tau_{i+1}-\tau_i)}.
	\]
    \end{lemma}

    \begin{proof}
	    Fix a phase $[\tau_i,\tau_{i+1})$ and for arm $a\in[K]$, let $\tau_i^a$ be the last round in phase $[\tau_i,\tau_{i+1})$ when $a$ is included in $\mc{S}_t$. WLOG, suppose $\tau_i^1 \leq \tau_i^2 \leq \cdots \leq \tau_i^K$. Then, the regret of a procedure $\pi$ which plays arm $a\in \mc{S}_t$ at round $t\in [\tau_i,\tau_{i+1})$ with probability $1/|\mc{S}_t|$ is:
	\[
		\sum_{t=\tau_i}^{\tau_{i+1}-1} \sum_{a\in \mc{S}_{t}} \frac{\delta_{t}(a)}{|\mc{S}_{t}|} \leq \sum_{a=1}^K \sum_{t=\tau_i}^{\tau_i^a-1} \frac{\delta_{t}(a)}{K+1-a} \leq \sum_{a=1}^K \frac{\sqrt{K(\tau_i^a - 1 - 	\tau_i)}}{K+1-a} \leq \log(K)\sqrt{K(\tau_{i+1} - \tau_i)},
	\]
	where we use the fact that $|\mc{S}_{t}| \geq K+1-a$ for $t < \tau_i^a$. Summing the regret over all phases $[\tau_i,\tau_{i+1})$ gives the desired result.
	\end{proof}

	The proof of Proposition~\ref{prop:sanity} follows by taking $\mc{S}_t=\mc{G}_t$ in Lemma~\ref{lem:combine}.

\section{Details for the Proof of Theorem~\ref{thm}}\label{app:full-proof}

\subsection{Bounding $\mb{E}[\sum_{t=t_{\ell}}^{t_{\ell+1}-1} \delta_t(a_{\ell},\pi_t)]$ (Term \ref{item:regret-ell} of Equation \eqref{eq:episode_decomposition})
}\label{subsec:2}
Following the discussion of Section~\ref{subsec:episode-regret}, we first reduce the problem to bounding the relative regret $\sum_t \delta_t(a',a)$. To relate the quantity $\delta_t(a_{\ell},\pi_t)$ to the relative gap $\delta_t(a_{\ell},a)$ for a fixed arm $a$, we first convert \ref{item:regret-ell} to an alternative form involving the relative regrets to fixed arms.

\begin{prop}
	\begin{equation}\label{eq:condition-arms}
	\mb{E}\left[ \sum_{t=t_{\ell}}^{t_{\ell+1}-1} \delta_t(a_{\ell},\pi_t) \right] =\mb{E}\left[ \sum_{t=t_{\ell}}^{t_{\ell+1}-1} \sum_{a\in\mc{A}_t} \frac{\delta_t(a_{\ell},a)}{|\mc{A}_t|}\right].
	\end{equation}
\end{prop}

\begin{proof}
This indeed follows from first conditioning on on $t_{\ell}$ and carefully applying tower property:
\begin{align*}
	\mb{E}\left[ \sum_{t=t_{\ell}}^{t_{\ell+1}-1} \delta_t(a_{\ell},\pi_t) \right] &= \mb{E}_{t_{\ell}}\left[ \mb{E}\left[ \sum_{t=t_{\ell}}^{t_{\ell+1}-1} \delta_t(a_{\ell},\pi_t) \mid t_{\ell} \right] \right]\\
										       &= \mb{E}_{t_{\ell}}\left[ \sum_{t=t_{\ell}}^T \mb{E}\left[ \delta_t(a_{\ell},\pi_t)\cdot \pmb{1}\{t < t_{\ell+1}\} \mid t_{\ell}\right] \right]\\
										       &= \mb{E}_{t_{\ell}}\left[ \sum_{t=t_{\ell}}^T \mb{E}[ \pmb{1}\{t<t_{\ell+1}\} \cdot \mb{E}[ \delta_t(a_{\ell},\pi_t) \mid \mc{F}_{t-1},t_{\ell}] \mid t_{\ell}] \right],
\end{align*}
where we use the fact that $\pmb{1}\{t<t_{\ell+1}\}$ is constant conditional on $\mc{F}_{t-1}$. Next, the innnermost expectation on the above R.H.S. is:
\[
	\mb{E}[\delta_t(a_{\ell},\pi_t) \mid \mc{F}_{t-1},t_{\ell}] = \sum_{a\in\mc{A}_t} \delta_t(a_{\ell},a)\cdot \mb{P}(\pi_t\text{ chooses }a \mid \mc{F}_{t-1},t_{\ell}) = \sum_{a\in\mc{A}_t} \frac{\delta_t(a_{\ell},a)}{|\mc{A}_t|}.
\]
Plugging the above into our earlier chain of expectations and unconditioning then gives us \eqref{eq:condition-arms}.
\end{proof}

%
%
Next, we condition on the good event $\mc{E}_1$ on which recall the concentration bounds of Proposition~\ref{prop:concentration} hold. We also further decompose \eqref{eq:condition-arms} by partitioning the rounds $t$ to those before arm $a$ is evicted from $\mc{A}_{\text{master}}$ and those after. Suppose arm $a$ is evicted from $\mc{A}_{\text{master}}$ at round $t_{\ell}^a\in [t_{\ell},t_{\ell+1})$. In particular, this means arm $a\in\mc{A}_t$ for all $t \in [t_{\ell},t_{\ell}^a)$. Thus, it suffices to bound:
\begin{equation}\label{eq:before-master-evict}
	\mb{E}\left[ \pmb{1}\{\mc{E}_1\}\cdot \left( \sum_{a=1}^K \sum_{t=t_{\ell}}^{t_{\ell}^a-1} \frac{\delta_t(a_{\ell},a)}{|\mc{A}_t|} + \sum_{a=1}^K \sum_{t=t_{\ell}^a}^{t_{\ell+1}-1}  \frac{\delta_t(a_{\ell},a)}{|\mc{A}_t|} \cdot \pmb{1}\{a\in \mc{A}_t\} \right) \right].
\end{equation}
Suppose WLOG that $t_{\ell}^1\leq t_{\ell}^2 \leq \cdots \leq t_{\ell}^K$. Then, for each round $t < t_{\ell}^a$ all arms $a'\geq a$ are retained in $\mc{A}_{\text{master}}$ and thus retained in the candidate arm set $\mc{A}_t$. Thus, $|\mc{A}_t| \geq K+1-a$ for all $t\leq t_{\ell}^a$.

Next, we bound the first double sum in \eqref{eq:before-master-evict}, i.e. the regret of playing $a$ to $a_{\ell}$ from $t_{\ell}$ to $t_{\ell}^a-1$, Applying Proposition~\ref{prop:concentration}, since arm $a$ is not evicted from $\mc{A}_t$ till round $t_{\ell}^a$, on event $\mc{E}_1$ we have for some $c_5>0$ and any other arm $a'\in \Amaster$ through round $t_{\ell}^a-1$ (i.e., $a'\in \mc{A}_t$ for all $t\in [t_{\ell},t_{\ell}^a)$):
\[
	\sum_{t=t_{\ell}}^{t_{\ell}^a-1} \mb{E}[ \hat{\delta}_t(a',a)\mid \mc{F}_{t-1}] \leq c_5 \log(T)\sqrt{K(t_{\ell}^a-t_{\ell})\vee K^2}.
\]
Next, since $a,a'\in \mc{A}_t$ for each $t\in [t_{\ell},t_{\ell}^a)$, we have:
\[
	\forall t\in [t_{\ell},t_{\ell}^a):\mb{E}[ \hat{\delta}_t(a',a) \mid \mc{F}_{t-1}] = \delta_t(a',a).
\]
Thus, we conclude for any such $a'$:
\[
	\sum_{t=t_{\ell}}^{t_{\ell}^a-1} \delta_t(a',a) \leq c_5\log(T)\sqrt{K(t_{\ell}^a - t_{\ell}) \vee K^2} \implies \sum_{t=t_{\ell}}^{t_{\ell}^a-1} \frac{\delta_t(a',a)}{|\mc{A}_t|} \leq \frac{c_5\log(T)\sqrt{K(t_{\ell}^a-t_{\ell})\vee K^2}}{K+1-a},
\]
where the second inequality appeals to $|\mc{A}_t|\geq K+1-a$ for all $t\in [t_{\ell},t_{\ell}^a)$. Since this last bound holds uniformly for all $a'\in\Amaster$ at round $t_{\ell}^a-1$, it must hold for the last master arm $a_{\ell}$. In particular,
\[
	\sum_{t=t_{\ell}}^{t_{\ell}^a-1} \frac{\delta_t(a_{\ell},a)}{|\mc{A}_t|} \leq \max_{a'\in \mc{A}_{t_{\ell}^a-1}} \sum_{t=t_{\ell}}^{t_{\ell}^a-1} \frac{\delta_t(a',a)}{|\mc{A}_t|} \leq \frac{c_5\log(T)\sqrt{K(t_{\ell}^a-t_{\ell})\vee K^2}}{K+1-a}.
\]
Then, summing the above R.H.S. over all arms $a$, we have on event $\mc{E}_1$:
\[
	\sum_{a=1}^K \sum_{t=t_{\ell}}^{t_{\ell}^a-1} \frac{\delta_t(a_{\ell},a)}{|\mc{A}_t|} \leq c_5\log(K)\log(T)\sqrt{K(t_{\ell+1}-t_{\ell}) \vee K^2}.
\]

Next, we handle the second double sum in \eqref{eq:before-master-evict}. We first observe that if arm $a$ is played after round $t_{\ell}^a$, then it must due to an active replay. The difficulty here is that replays may interrupt each other and so care must be taken in managing the relative regret contribution $\sum_t\delta_t(a_{\ell},a)$ (which may be negative) of different overlapping replays.

%
%

Our strategy is to partition the rounds when a given arm $a$ is played by a replay after round $t_{\ell}^a$ according to which replay is active and not accounted for by another replay. This will allow us to isolate the relative regret contributions $\delta_t(a_{\ell},a)$ to \eqref{eq:before-master-evict} of a given replay.

For this purpose, we define the following notation.

\begin{defn}
Recall that the Bernoulli $B_{s,m}$ (see Line~\ref{line:add-replay} of Algorithm~\ref{meta-alg}) decides whether $\base(s,m)$ is scheduled (and hence has the chance to become active at all).

Call a replay $\base(s,m)$ {\bf proper} if there is no other activated replay $\base(s',m')$ such that $[s,s+m] \subset (s',s'+m')$ where $\base(s',m')$ will become active again after round $s+m$. In other words, a proper replay is not scheduled inside the scheduled range of rounds of another replay. For each $\base(s,m)$, let $M(s,m,a)$ be the last round in $[s,s+m]$ when arm $a$ is retained by $\base(s,m)$ and all of its children.
Furthermore, let $\textsc{Proper}(t_{\ell},t_{\ell+1})$ be the set of proper replays scheduled to start before round $t_{\ell+1}$. Let $\mc{R}(a)$ be the set of scheduled replays $\base(s,m)$ such that its parent $\base$ has evicted arm $a$ before round $s$.
\end{defn}

\begin{figure}
	\centering
	\includegraphics[scale=0.5]{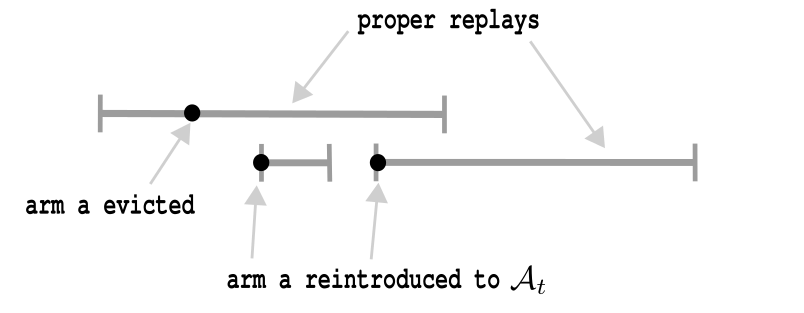}
	\caption{\noindent \small
		Shown are various replay scheduled durations (gray segments) with indications of when a given arm $a$ may be reintroduced to $\mc{A}_t$.
	}
	\label{fig:proper}
\end{figure}

We first observe that for each round $t$ when a replay is active, there is a unique proper replay associated to $t$. This is either the replay active at round $t$ or the proper replay which was scheduled most recently before $t$.
Then, for a fixed arm $a$, we partition the set of rounds $t>t_{\ell}^a$ where arm $a\in\mc{A}_t$ (and hence could be played) into intervals, each initiated by the start of a unique $\base(s,m)\in \textsc{Proper}(t_{\ell},t_{\ell+1})\cup \mc{R}(a)$ and terminated when arm $a$ is evicted or when the replay in question terminates, i.e. at round $M(s,m,a)$. Note that in this partition of rounds, it is not enough to merely consider the scheduling of proper replays $\base(s,m)\in \textsc{Proper}(t_{\ell},t_{\ell+1})$ since non-proper children of proper replays may resample arm $a$ after their parent evicts $a$. It is also not enough to consider just the schedules of replays in $\mc{R}(a)$, which reintroduce arm $a$, since a proper replay may not evict $a$ before being pre-empted by another proper replay. This delicate accounting of the rounds where arm $a$ is resampled is demonstrated in Figure~\ref{fig:proper}.

Then, the second double sum in \eqref{eq:before-master-evict} can be decomposed into a sum over such replays and then a sum over the rounds where each such replay retains arm $a$. In other words, we can write the second double sum of \eqref{eq:before-master-evict} as:
\[
	\sum_{a=1}^K \sum_{\base(s,m)\in\textsc{Proper}(t_{\ell},t_{\ell+1})\cup \mc{R}(a)} \pmb{1}\{B_{s,m}=1\} \sum_{t=s\vee t_{\ell}^a}^{M(s,m,a)} \frac{\delta_t(a_{\ell},a)}{|\mc{A}_t|}.
\]
Further bounding the sum over $t$ above by its positive part, we can sum over all $\base(s,m)$, or obtain:
\begin{equation}\label{eq:positive-part}
	\sum_{a=1}^K \sum_{\base(s,m)} \pmb{1}\{B_{s,m}=1\}\left( \sum_{t=s\vee t_{\ell}^a}^{M(s,m,a)} \frac{\delta_t(a_{\ell},a)}{|\mc{A}_t|}\cdot\pmb{1}\{a\in\mc{A}_t\} \right)_+ ,
\end{equation}

where the sum is over all replays $\base(s,m)$, i.e. $s\in \{t_{\ell}+1,\ldots,t_{\ell+1}-1\}$ and $m\in \{2,4,\ldots,2^{\lceil \log(T)\rceil}\}$.
It then remains to bound the contributed relative regret of each $\base(s,m)$ in the interval $[s\vee t_{\ell}^a,M(s,m,a)]$, which will follow similarly to the previous steps. Fix $s,m$ and suppose $t_{\ell}^a+1 \leq M(s,m,a)$ since otherwise $\base(s,m)$ contributes no regret in \eqref{eq:positive-part}.

Then, following similar reasoning as before, i.e. combining our concentration bound \eqref{eq:error-bound} with the eviction criterion \eqref{eq:elim}, we have for a fixed arm $a$:
\[
	\sum_{t=s \vee t_{\ell}^{a}}^{M(s,m,a)} \frac{\delta_t(a_{\ell},a)}{|\mc{A}_t|} \leq \frac{c_5\log(T)\sqrt{Km \vee K^2}}{\min_{t\in [s,M(s,m,a)]} |\mc{A}_t|},
\]

Plugging this into \eqref{eq:positive-part} and switching the ordering of the outer double sum, we obtain
\[
	\sum_{\base(s,m)} c_5\log(T)\sqrt{Km \vee K^2} \sum_{a=1}^K \frac{1}{\min_{t\in [s,M(s,m,a)]} |\mc{A}_t|}.
\]
We claim this inner sum over $a$ is at most $\log(K)$. For a fixed $\base(s,m)$, if $a_k$ is the $k$-th arm in $[K]$ to be evicted by $\base(s,m)$ or any of its children, then $\min_{t\in [s,M(s,m,a_k)]} |\mc{A}_t| \geq K+1-k$. Thus, our claim follows follows from $\sum_{k=1}^K \frac{1}{K+1-k} \leq \log(K)$.

Let $R(m) \doteq  c_5\log(K)\log(T)\sqrt{Km\vee K^2}$ which is the bound we've obtained so far on the relative regret for a single $\base(s,m)$. Then,  plugging $R(m)$ into \eqref{eq:positive-part} gives:
\begin{align*}
	\mb{E}&\left[ \pmb{1}\{\mc{E}_1\} \sum_{a=1}^K \sum_{t=t_{\ell}^a}^{t_{\ell+1}-1}  \frac{\delta_t(a_{\ell},a)}{|\mc{A}_t|} \cdot \pmb{1}\{a\in \mc{A}_t\} \right] \leq \mb{E}_{t_{\ell}}\left[ \mb{E}\left[ \sum_{\base(s,m)} \pmb{1}\{B_{s,m}=1\}\cdot R(m) \mid t_{\ell} \right] \right]\\
																					  &=  \mb{E}_{t_{\ell}}\left[ \sum_{s=t_{\ell+1}}^T \sum_m  \mb{E}[ \pmb{1}\{B_{s,m}=1\}\cdot \pmb{1}\{s<t_{\ell+1}\} \mid t_{\ell}] \cdot R(m) \right].
\end{align*}
Next, we observe that $B_{s,m}$ and $\pmb{1}\{s<t_{\ell+1}\}$ are independent conditional on $t_{\ell}$ since $\pmb{1}\{s<t_{\ell+1}\}$ only depends on the scheduling and observations of base algorithms scheduled before round $s$. Thus, recalling that $\mb{P}(B_{s,m}=1) = 1/\sqrt{m\cdot (s-t_{\ell})}$,
\begin{align*}
	\mb{E}[ \pmb{1}\{B_{s,m}=1\}\cdot \pmb{1}\{s<t_{\ell+1}\} \mid t_{\ell}] &=  \mb{E}[ \pmb{1}\{B_{s,m}=1\} \mid t_{\ell}] \cdot \mb{E}[ \pmb{1}\{s<t_{\ell+1}\} \mid t_{\ell}]\\
										 &= \frac{1}{\sqrt{m\cdot (s-t_{\ell})}} \cdot \mb{E}[ \pmb{1}\{ s < t_{\ell+1}\} \mid t_{\ell}].
\end{align*}
Plugging this into our expectation from before and unconditioning, we obtain:
\begin{equation}\label{eq:regret-replay}
	\mb{E}\left[ \sum_{s=t_{\ell}+1}^{t_{\ell+1}-1} \sum_{n=1}^{\lceil \log(T)\rceil} \frac{1}{\sqrt{2^n \cdot (s - t_{\ell})}} \cdot R(2^n) \right] \leq c_6\log^3(T) \mb{E}_{t_{\ell},t_{\ell+1}} \left[  \sqrt{K(t_{\ell+1} - t_{\ell}) \vee K^2}\right].
\end{equation}
Then, to bound \ref{item:regret-ell}, it suffices to bound $\sqrt{K(t_{\ell+1}-t_{\ell})\vee K^2}$. First, we claim that every phase $[\tau_i,\tau_{i+1})$ is length at least $K/4$. Observe by our notion of significant regret, that an arm $a$ incurring significant regret on the interval $[s_1,s_2]$ means
\[
	\sum_{t=s_1}^{s_2} \delta_t(a) \geq \sqrt{K\cdot (s_2-s_1)} \implies 2\cdot (s_2-s_1) \geq \sqrt{K\cdot (s_2 - s_1)} \implies s_2 - s_1 \geq K/4.
\]
Thus, each significant phase (Definition~\ref{defn:sig-shift}) must be at least $K/4$ rounds long meaning $\tau_{i+1}-\tau_i = (\tau_{i+1} - \tau_i) \vee K/4$.

Showing \ref{item:regret-ell} is order \eqref{eq:regret-episode} then follows from conditioning again on $\mc{E}_1$ in \eqref{eq:regret-replay} and upper bounding $t_{\ell+1}-t_{\ell}$ by the combined length of all phases $[\tau_i,\tau_{i+1})$ intersecting episode $[t_{\ell},t_{\ell+1})$.

\subsection{Bounding $\mb{E}[ \sum_{t=t_{\ell}}^{t_{\ell+1}-1} \delta_t(a_t^\sharp,a_{\ell})]$ (Term \ref{item:regret-persistent} of Equation \eqref{eq:episode_decomposition})}\label{subsec:3}

	To show Proposition~\ref{cor:behavior}, we first need an elementary lemma which only depends on the definition of significant shift (Definition~\ref{defn:sig-shift}).

	\begin{lemma}\label{lem:detect}
		Let $[s_{i,j}(a),s_{i,j+1}(a))$ be a bad segment, defined with respect to arm $a$. Then
		\begin{equation}\label{eq:relative-gap-detect}
			\sum_{t=\left\lceil\frac{s_{i,j}(a)+s_{i,j+1}(a)}{2}\right\rceil}^{s_{i,j+1}(a)} \delta_{t}(a_i^\sharp,a) \geq \frac{c_4}{4} \log(T)\sqrt{ K\left( s_{i,j+1}(a) - \left(\frac{s_{i,j}(a)+s_{i,j+1}(a)}{2}\right)\right)}.
		\end{equation}
	\end{lemma}

				\begin{proof}{(of Lemma~\ref{lem:detect})}
					Let $\tilde{s} \doteq \left\lceil\frac{s_{i,j}(a) + s_{i,j+1}(a)}{2}\right\rceil$ be the midpoint between $s_{i,j}(a)$ and $s_{i,j+1}(a)$. Then, we have by  \eqref{eq:segment} in the construction of the $s_{i,j}(a)$'s (Definition~\ref{defn:bad-segment}) that:
		\begin{align*}
			\sum_{t = \tilde{s}}^{s_{i,j+1}(a)} \delta_{t}(a_i^\sharp,a) &= 
			\sum_{t = s_{i,j}(a)}^{s_{i,j+1}(a)}  \delta_{t}(a_i^\sharp,a)
			- \sum_{t=s_{i,j}(a)}^{\tilde{s}-1} \delta_t(a_i^\sharp,a)\\
									     &\geq c_4 \log(T)\sqrt{K}\left(\sqrt{s_{i,j+1}(a) - s_{i,j}(a)} - \sqrt{\tilde{s} -1 - s_{i,j}(a)}\right) \\
									     &\geq \frac{c_4}{4}\log(T) \sqrt{K (s_{i,j+1}(a) - \tilde{s})}, \numberthis\label{eq:elem-ineq}
		\end{align*}
		where the last step is from the elementary fact $\sqrt{a+b}-\sqrt{a} \geq \sqrt{b}/4$ for any $b\geq a \geq 0$.
	\end{proof}

	\begin{proof}{(of Proposition~\ref{cor:behavior})}
		Suppose event $\mc{E}_1$ (i.e., our concentration bound \eqref{eq:error-bound}) holds. For \ref{2a}, if $\base(\tstart,m)$ with $\tstart \geq s_{i,j}(a)$ evicts arm $a_i^\sharp$ before round $s_{i,j+1}(a)$, then arm $a_i^\sharp$ incurs significant regret \eqref{eq:bad-arm} on a subinterval of $[s_{i,j}(a),s_{i,j+1}(a)]$, which is a contradiction to the definition of arm $a_i^\sharp$.

For \ref{2b}, we first observe $\mb{E}[\hat{\delta}_t(a_i^\sharp,a) \mid \mc{F}_{t-1}] = \delta_t(a_i^\sharp,a)$ for any round $t\in [\tilde{s},s_{i,j+1}(a)]$ if $a_i^\sharp,a\in\mc{A}_t$. Then, combining \eqref{eq:relative-gap-detect} of Lemma~\ref{lem:detect} with our concentration bound \eqref{eq:error-bound}, we have that arm $a$ will satisfy the eviction criterion \eqref{eq:elim} over interval $[\tilde{s},s_{i,j+1}(a)]$.
	\end{proof}

\ifx
	\begin{proof}{(of Proposition~\ref{prop:perfect})}
		Since $\base(s,m)$ becomes active at round $s\in [s_{i,j}(a),\tilde{s}]$, by \ref{2a} of Corollary~\ref{cor:behavior}, this base algorithm (and any child of it) will not evict arm $a_i^\sharp$ before round $s_{i,j+1}(a)+1$. Thus, $a_i^\sharp\in \mc{A}_t$ for all rounds $t\in [\tilde{s},s_{i,j+1}(a))$ meaning by \ref{2b} of Proposition~\ref{cor:behavior}, arm $a$ will be excluded from $\mc{A}_{\text{master}}$ by round $s_{i,j+1}(a)$.
	\end{proof}
\fi

	Next, following the outline of Section~\ref{subsec:episode-regret}, we bound the the regret of a fixed arm $a$ to $a_i^\sharp$ over the bad segments w.r.t. $a$. It should be understood that in what follows, we condition on $t_{\ell}$. First, fix an arm $a$ and define the {\em bad round} $s(a)>t_{\ell}$ as the smallest round which satisfies, for some fixed $c_7>0$:
		\begin{equation}\label{eq:big-segment-regret}
			\ds\sum_{(i,j)}  \sqrt{s_{i,j+1}(a)-s_{i,j}(a)} > c_7\log(T) \sqrt{s(a) - t_{\ell}},
		\end{equation}
		where the above sum is over all pairs of indices $(i,j)\in\mb{N}\times\mb{N}$ such that $[s_{i,j}(a),s_{i,j+1}(a))$ is a bad segment with $s_{i,j+1}(a)<s(a)$. We will show that arm $a$ is evicted within episode $\ell$ with high probability by the time the bad round $s(a)$ occurs.


		For each bad segment $[s_{i,j}(a),s_{i,j+1}(a))$, let $\tilde{s}_{i,j}(a) \doteq \left\lceil \frac{s_{i,j}(a)+s_{i,j+1}(a)}{2} \right\rceil$ denote the midpoint of the bad segment and also let $m_{i,j} \doteq 2^n$ where $n\in\mb{N}$ satisfies:
		\[
			2^n \geq s_{i,j+1}(a) - s_{i,j}(a) > 2^{n-1}.
		\]
		Next, recall that the Bernoulli $B_{m.t}$ decides whether $\base(t,m)$ activates at round $t$ (see Line~\ref{line:add-replay} of Algorithm~\ref{meta-alg}). If for some $t\in [s_{i,j}(a),\tilde{s}_{i,j}(a)]$, $B_{t,m_{i,j}}=1$, i.e. a perfect replay is scheduled, then $a$ will be evicted from $\Amaster$ by round $s_{i,j+1}(a)$ (Corollary~\ref{cor:perfect}). We will show this happens with high probability via concentration on the sum $\sum_{(i,j)}\sum_t B_{t,m_{i,j}}$ where $j,i,t$ run through all $t\in [s_{i,j}(a),\tilde{s}_{i,j}(a))$ and all bad segments $[s_{i,j}(a),s_{i,j+1}(a))$ with $s_{i,j+1}(a)<s(a)$. Note that these random variables only depend on the fixed arm $a$, the episode start time $t_{\ell}$, and the randomness of scheduling replays on Line~\ref{line:add-replay}. In particular, the $B_{t,m_{i,j}}$ are independent conditional on $t_{\ell}$.

		Then, a Chernoff bound over the randomization of \meta on Line~\ref{line:add-replay} of Algorithm~\ref{meta-alg} conditional on $t_{\ell}$ yields
		\[
		\mb{P}\left( \sum_{(i,j)}\sum_{t} B_{t,m_{i,j}} \leq \frac{\mb{E}[ \sum_{(i,j)}\sum_{t} B_{t,m_{i,j}}\mid t_{\ell}]}{2} \mid t_{\ell}\right) \leq \exp\left(-\frac{\mb{E}[\sum_{(i,j)}\sum_{t}  B_{t,m_{i,j}} \mid t_{\ell}]}{8}\right).
		\]
		We claim the error probability on the R.H.S. above is at most $1/T^3$. To this end, we compute:
		\begin{align*}
			\mb{E}\left[ \sum_{(i,j)}\sum_{t} B_{t,m_{i,j}} \mid t_{\ell}\right] &\geq \ds\sum_{(i,j)}  \sum_{t=s_{i,j}(a)}^{\tilde{s}_{i,j}(a)} \frac{1}{\sqrt{m_{i,j}\cdot (t - t_{\ell})}}
												  \geq \frac{1}{4}\ds\sum_{(i,j)} \sqrt{\frac{s_{i,j+1}(a)-s_{i,j}(a)}{s(a) - t_{\ell}}}
												   \geq \frac{c_7}{4} \log(T),
		\end{align*}
		where the last inequality follows from \eqref{eq:big-segment-regret}. The R.H.S. above is larger than $24\log(T)$ for $c_7$ large enough, showing that the error probability is small. Taking a further union bound over the choice of arm $a\in[K]$ gives us that $\sum_{(i,j)}\sum_{t} B_{t,m_{i,j}} > 1$ for all choices of arm $a$ (define this as the good event $\mc{E}_2(t_{\ell})$) with probability at least $1-K/T^3$.

		Recall on the event $\mc{E}_1$ the concentration bounds of Proposition~\ref{prop:concentration} hold. Then, on $\mc{E}_1\cap \mc{E}_2(t_{\ell})$, we must have $t_{\ell+1}-1 \leq s(a_{\ell})$ since otherwise $a_{\ell}$ would have been evicted by some perfect replay before the end of the episode $t_{\ell+1}-1$ by virtue of $\sum_{(i,j)}\sum_t B_{t,m_{i,j}}>1$ for arm $a_{\ell}$. Thus, by the definition of the bad round $s(a_{\ell})$ \eqref{eq:big-segment-regret}, we must have:
		\begin{equation}\label{eq:not-too-many-bad}
			\ds\sum_{[s_{i,j}(a_{\ell}),s_{i,j+1}(a_{\ell})): s_{i,j+1}(a_{\ell}) < t_{\ell+1}-1}  \sqrt{s_{i,j+1}(a_{\ell})-s_{i,j}(a_{\ell})} \leq c_7\log(T) \sqrt{t_{\ell+1} - t_{\ell}}.
		\end{equation}
		Thus, by \eqref{eq:segment} in Definition~\ref{defn:bad-segment}, over the bad segments $[s_{i,j}(a_{\ell}),s_{i,j+1}(a_{\ell}))$ which elapse before the end of the episode $t_{\ell+1}-1$, the regret of $a_{\ell}$ to $a_t^\sharp$ is at most order $\log^2(T)\sqrt{K\cdot (t_{\ell+1}-t_{\ell})}$.

	Over each non-bad critical segment $[s_{i,j}(a_{\ell}),s_{i,j+1}(a_{\ell}))$ and the last segment $[s_{i,j}(a_{\ell}),t_{\ell+1}-1)$, the regret of playing arm $a_{\ell}$ to $a_i^\sharp$ is at most $\log(T)\sqrt{\tau_{i+1}-\tau_i}$ since there is at most one non-bad critical segment per phase $[\tau_i,\tau_{i+1})$ (see \eqref{eq:segment} in Definition~\ref{defn:bad-segment}).

		So, we conclude that on event $\mc{E}_1\cap \mc{E}_2(t_{\ell})$:
		\[
			\sum_{t=t_{\ell}}^{t_{\ell+1}-1} \delta_t(a_t^\sharp,a_{\ell}) \leq c_8\log^{2}(T)\sum_{i\in\textsc{Phases}(t_{\ell},t_{\ell+1})} \sqrt{K(\tau_{i+1} - \tau_i)}.
		\]
		Taking expectation, we have by conditioning first on $t_{\ell}$ and then on event $\mc{E}_1\cap \mc{E}_2(t_{\ell})$:
		\begin{align*}
			\mb{E}\left[ \sum_{t=t_{\ell}}^{t_{\ell+1}-1} \delta_t(a_t^\sharp,a_{\ell})\right] &\leq  \mb{E}_{t_{\ell}} \left[ \mb{E}\left[ \pmb{1}\{\mc{E}_1\cap \mc{E}_2(t_{\ell})\} \sum_{t=t_{\ell}}^{t_{\ell+1}-1} \delta_t(a_t^\sharp,a_{\ell}) \mid t_{\ell} \right] \right] + T\cdot \mb{E}_{t_{\ell}} \left[ \mb{E}\left[ \pmb{1}\{\mc{E}_1^c \cup \mc{E}_2^c(t_{\ell})\} \mid t_{\ell} \right] \right]\\
												      &\leq c_8\log^2(T) \mb{E}_{t_{\ell}} \left[ \mb{E}\left[ \pmb{1}\{\mc{E}_1\cap \mc{E}_2(t_{\ell})\} \sum_{i\in\textsc{Phases}(t_{\ell},t_{\ell+1})} \sqrt{K (\tau_{i+1} - \tau_i}) \mid t_{\ell} \right]\right] + \frac{K}{T^2}\\
												      &\leq c_8\log^2(T) \mb{E}\left[ \pmb{1}\{\mc{E}_1\}  \sum_{i\in\textsc{Phases}(t_{\ell},t_{\ell+1})} \sqrt{\tau_{i+1} - \tau_i}\right]  + \frac{1}{T},
		\end{align*}
		where in the last step we bound $\pmb{1}\{\mc{E}_1\cap \mc{E}_2(t_{\ell})\} \leq \pmb{1}\{\mc{E}_1\}$ and apply tower law again. This shows \ref{item:regret-persistent} is at most order the R.H.S. of \eqref{eq:regret-episode}.	$\hfill\blacksquare$



\section{Proof of Corollary~\ref{cor:tv}}\label{app:tv-proof}

The proof of Corollary~\ref{cor:tv} follows straightforwardly from Definition~\ref{defn:sig-shift}. Recall from Section~\ref{subsec:result} that $V \doteq \sum_{t = 2}^T \max_{a\in [K]} |\mu_{t}(a) - \mu_{t - 1}(a)|$ is the total variation of change in the rewards. Then, by Theorem~\ref{thm}, it suffices to show
\begin{equation}\label{eq:variation-bound}
	\sum_{i=0}^{\Lsig} \sqrt{K(\tau_{i+1} - \tau_i)} \lesssim \sqrt{KT} + (KV)^{1/3}\cdot T^{2/3}.
\end{equation}
Fix a phase $[\tau_i,\tau_{i+1})$ such that $\tau_{i+1}<T+1$. We first bound the total variation over this phase \[
	V_{[\tau_i,\tau_{i+1})} \doteq \sum_{t = \tau_i+1}^{\tau_{i+1}} \max_{a\in [K]} |\mu_{t}(a) - \mu_{t-1}(a)|.
\]
By the definition of significant shift (Definition~\ref{defn:sig-shift}), an arm $a=\argmax_{a\in[K]} \mu_{\tau_{i+1}}(a)$ must incur significant regret \eqref{eq:bad-arm} on the interval $[s_1,s_2]$ for some $\tau_i \leq s_1 < s_2 < \tau_{i+1}$, or
\[
	\sum_{t=s_1}^{s_2} \delta_{t}(a) \geq \sqrt{K(s_2 - s_1)}.
\]
Since $\sqrt{s_2-s_1} > \sum_{t=s_1}^{s_2} 1/\sqrt{\tau_{i+1} - \tau_i}$, there must be a round $t \in [s_1,s_2]$ such that $\delta_t(a) \geq \sqrt{K/(\tau_{i+1} - \tau_i)}$. Let $a' \in \argmax_{a'\in[K]} \mu_{t}(a')$. Then, we have
\begin{align*}
	\sqrt{\frac{K}{\tau_{i+1} - \tau_i}} \leq \delta_t(a) &\leq \mu_t(a') - \mu_t(a) + \mu_{\tau_{i+1}}(a) - \mu_{\tau_{i+1}}(a')\\
				       &\leq |\mu_{\tau_{i+1}}(a) - \mu_t(a)| + |\mu_t(a') - \mu_{\tau_{i+1}}(a')|\\
				       &\leq 2\sum_{s=t+1}^{\tau_{i+1}} \max_{a\in [K]} |\mu_{s}(a) - \mu_{s-1}(a)|\\
				       &\leq 2\cdot V_{[\tau_i,\tau_{i+1})}
\end{align*}
This gives us a lower bound on the total variation over each interval $[\tau_i,\tau_{i+1})$. Then, by H\"{o}lder's inequality:
\begin{align*}
	\sum_{i=0}^{\Lsig} \sqrt{K(\tau_{i+1} - \tau_i)} &= \sqrt{K(T+1-\tau_{\Lsig})} + \sum_{i=0}^{\tilde{L}-1} \sqrt{K(\tau_{i+1}-\tau_i)}\\
							 &\leq \sqrt{KT}+ \left(\sum_{i=0}^{\Lsig} \sqrt{\frac{K}{\tau_{i+1} - \tau_i}}\right)^{1/3} \left(\sum_{i=0}^{\Lsig} (\tau_{i+1} - \tau_i)\sqrt{K}\right)^{2/3} \\
							 &\leq \sqrt{KT} + \left(\sum_{i=0}^{\Lsig} 2V_{[\tau_i,\tau_{i+1})}\right)^{1/3} K^{1/3}T^{2/3} \\
							 &= \sqrt{KT} + (2KV)^{1/3}\cdot T^{2/3}.
\end{align*}
$\hfill\blacksquare$

\end{document}